\documentclass[twoside]{article}

\usepackage[accepted]{aistats2019}
\usepackage{xcolor}

\usepackage{booktabs}       
\usepackage{epsfig}         
\usepackage{float}          
\usepackage{graphicx}       
\usepackage{hyperref}       

\usepackage{microtype}      
\usepackage{subfig}
\usepackage{url}            
\usepackage{xspace}
\usepackage{amssymb}

\usepackage{multirow}
\usepackage{enumitem}
 \usepackage[ruled,vlined,noresetcount]{algorithm2e}
 \usepackage{algpseudocode}
 \usepackage{amsthm}
 
\usepackage{wrapfig}

\allowdisplaybreaks

\newtheorem{Thm}{Theorem}
\newtheorem*{Thm*}{Theorem}
\newtheorem{Lm}{Lemma}
\newtheorem{Df}{Definition}

\newtheorem{Lem}[Thm]{Lemma}

\newtheorem{Prop}[Thm]{Proposition}

\newtheoremstyle{TheoremNum}
    {\topsep}{\topsep}              
    {\itshape}                      
    {}                              
    {\bfseries}                     
    {.}                             
    { }                             
    {\thmname{#1}\thmnote{ \bfseries #3}}
\theoremstyle{TheoremNum}

 \newtheorem{innercustomthm}{Theorem}
\newenvironment{customthm}[1]
  {\renewcommand\theinnercustomthm{#1}\innercustomthm}
  {\endinnercustomthm}

\newcommand{\R}{\mathbb{R}}
\newcommand{\E}{\mathbb{E}}
\DeclareMathOperator{\argmin}{argmin}
\usepackage{bbm}
\begin{document}

%

%
\runningauthor{Jia, Dao, Wang, Hubis, Hynes, Guerel, Li, Zhang, Song, Spanos}

\twocolumn[

\aistatstitle{Towards Efficient Data Valuation Based on the Shapley Value}

\aistatsauthor{Ruoxi Jia$^{1*}$, David Dao$^{2*}$, Boxin Wang$^3$, Frances Ann Hubis$^{2}$, Nick Hynes$^{1}$, \\\textbf{Nezihe Merve Gurel$^{2}$, Bo Li$^{4}$, Ce Zhang$^{2}$, Dawn Song$^{1}$, Costas Spanos$^{1}$}}

\aistatsaddress{$^1$University of California at Berkeley, $^2$ETH, Zurich\\$^3$Zhejiang University, $^4$University of Illinois at Urbana-Champaign} 
]

\begin{abstract}
{\em ``How much is my data worth?''} is an increasingly common question posed by organizations and individuals alike. An answer to this question could allow, for instance, fairly distributing profits among multiple data contributors and determining prospective compensation when data breaches happen. In this paper, we study the problem of \emph{data valuation} by utilizing the Shapley value, a popular notion of value which originated in cooperative game theory. The Shapley value defines a unique payoff scheme that satisfies many desiderata for the notion of data value. However, the Shapley value often requires \emph{exponential} time to compute.
To meet this challenge, we propose a repertoire of efficient algorithms for approximating the Shapley value. We also demonstrate the value of each training instance for various benchmark datasets.
\end{abstract}

\section{Introduction}
Data analytics using machine learning (ML) is an increasingly common practice in modern science and business. The data for building an ML model are often provided by multiple entities. For instance, Internet enterprises analyze various users' data to improve product design, customer retention, and initiatives that help them earn revenue. Furthermore, the quality of the data from different entities may vary widely. Therefore, a key question often asked by stakeholders of a ML system is how to fairly allocate the revenue generated by a ML model to the data contributors.

\begin{figure}[t]
\centering
\includegraphics[width=1\columnwidth]{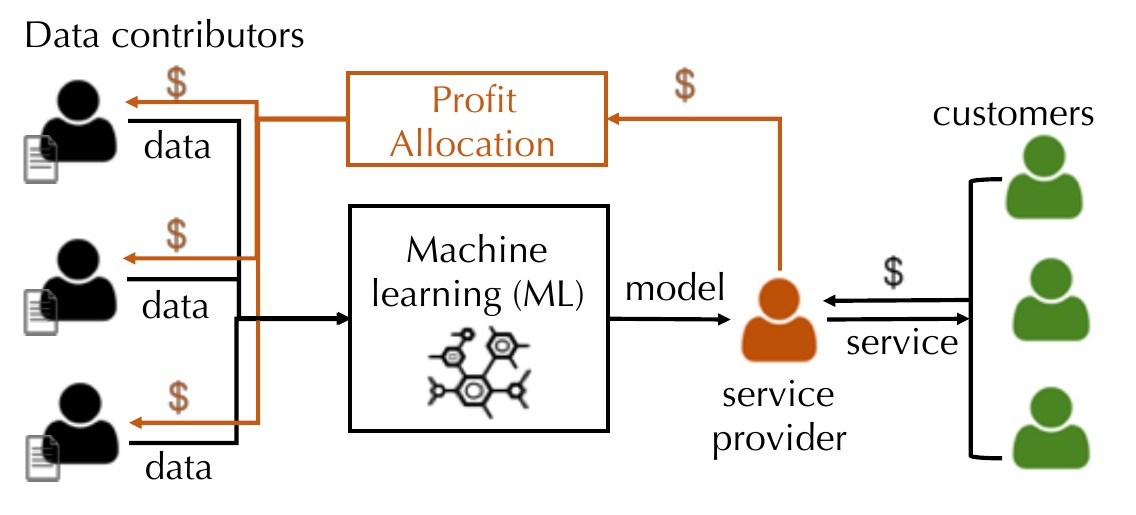}
\caption{Overview of the data valuation problem.}
\label{fig:overview}
\end{figure}

This question is also motivated by a system we are building together with one of the largest hospital in the US. In the system, patients submit part of their medical records onto a ``data market,'' and analysts pay a certain amount of money to train a ML model on patients' data. One of the challenges in such data markets is how to distribute the payment from analysts back to the patients.

A natural way of tackling the data valuation problem is to adopt a game-theoretic viewpoint, where each data contributor is modeled as a player in a coalitional game and the usefulness of data from any subset of contributors is characterized via a utility function. The Shapley value (SV) is a classic method in cooperative game theory to distribute the total gains generated by the coalition of all players, and has been applied to problems in various domains, ranging from economics~\cite{gul1989bargaining}, counter-terrorism~\cite{michalak2013computational,lindelauf2013cooperative}, environmental science~\cite{petrosjan2003time}, to ML~\cite{cohen2005feature}. The reason for its broad adoption is that the SV defines a unique profit allocation scheme that satisfies a set of properties with appealing real-world interpretations, such as fairness, rationality, and decentralizability.

Despite the desirable properties of the SV, computing the SV is known to be expensive; the number of  utility function evaluations required by the exact SV calculation grows exponentially in the number of players. This poses a radical challenge to using the SV in the context of data valuation---how to calculate, or approximate the SV over millions or even billions of data points, a scale that is rare in previous applications of the SV, but not uncommon for real-world data valuation tasks. Even worse, for ML tasks, evaluating the utility function itself (e.g., testing accuracy) is already computationally expensive, as it requires training a model. Due to the computational challenge, the application of the SV to data valuation has thus far been limited to stylized examples, in which the underlying utility function of the game is simple and the resulting SV can be represented as a closed-form expression~\cite{kleinberg2001value,chessa2017cooperative}. The state-of-the-art method to estimate the SV for a black-box utility function is based on Monte Carlo simulations~\cite{maleki2013bounding}, which still requires evaluating ML models for $\mathcal{O}(N^2\log N)$ many times in order to compute the SV of $N$ data points and is thus clearly impracticable.
In this paper, we attempt to answer the question of whether it is possible to efficiently estimate the SV while achieving the same performance guarantee as the state-of-the-art method. 

\begin{table*}[t!]
\caption{Summary of Technical Results. $N$ is the number of data points.}
\label{fig:summary_of_results}
\resizebox{\textwidth}{!}{
\begin{tabular}{cccclc}
\hline
\multicolumn{1}{l}{\multirow{2}{*}{}}                                                     & \multirow{2}{*}{\textbf{Assumptions}} & \multirow{2}{*}{\textbf{Techniques}}                                        & \multicolumn{2}{c}{\textbf{Complexity}}                                                                                                                                                                                                           & \multirow{2}{*}{\textbf{Approximation}} \\ \cline{4-5}
\multicolumn{1}{l}{}                                                                      &                                       &                                                                             & \textbf{incrementally trainable models}                                                                                                      & \multicolumn{1}{c}{\textbf{otherwise}}                                                             &                                         \\ \hline
\textbf{Existing}                                                                         & Bounded utility                       & Permutation sampling                                                        & \multicolumn{1}{l}{\begin{tabular}[c]{@{}l@{}}$\mathcal{O}(N\log N)$ model training \\ and $\mathcal{O}(N^2\log N)$ eval\end{tabular}}     & \begin{tabular}[c]{@{}l@{}}$\mathcal{O}(N^2\log N)$ \\ model training and eval\end{tabular}       & $(\epsilon, \delta)$                    \\ \hline
\multirow{2}{*}{\textbf{\begin{tabular}[c]{@{}c@{}}Application\\ -agnostic\end{tabular}}} & Bounded utility                       & Group testing                                                               & \multicolumn{1}{l}{\begin{tabular}[c]{@{}l@{}}$\mathcal{O}(N(\log N)^2)$ model training \\ and eval\end{tabular}}                      & \begin{tabular}[c]{@{}l@{}}$\mathcal{O}(N(\log N)^2)$\\ model training and eval\end{tabular} & $(\epsilon, \delta)$                    \\ \cline{2-6} 
                                                                                          & \begin{tabular}[c]{@{}c@{}}Monotone utility \& \\ sparse value\end{tabular}                          & \begin{tabular}[c]{@{}c@{}}Compressive \\ permutation sampling\end{tabular} & \multicolumn{1}{l}{\begin{tabular}[c]{@{}l@{}}$\mathcal{O}(\log\log N)$ model training\\ and $\mathcal{O}(N\log\log N)$ eval\end{tabular}} & \begin{tabular}[c]{@{}l@{}}$\mathcal{O}(N\log\log N)$ \\ model training and eval\end{tabular}     & $(\epsilon, \delta)$                    \\ \hline
\multirow{2}{*}{\textbf{ML-specific}}                                                     & Stable learning                       & Uniform division                                                            & \multicolumn{2}{c}{$\mathcal{O}(1)$ computation}                                                                                                                                                                                                  & $(\epsilon, 0)$                         \\ \cline{2-6} 
                                                                                          & Smooth utility                        & Influence function                                                          & \multicolumn{2}{c}{$\mathcal{O}(N)$ optimization routines}                                                                                                                                                                                        & Heuristic                               \\ \hline
\end{tabular}
}
\end{table*}

{\bf Theoretical Contribution} We first
study this question from a theoretical
perspective. We show that, to approximate
the SV of $N$ data points with provable error guarantees,
it is possible to
design an algorithm with
$\mathcal{O}(N(\log N)^2)$ model evaluations\footnote{See the technique note by Wang and Jia \cite{wang2023note} for an improved version of this algorithm.}.
We achieve this by enabling proper information sharing between different model evaluations. Moreover, if it is reasonable to assume that the utility function is monotone and the SV is ``sparse'' in the sense that only few data points have significant values, then we are able to further reduce the number of model training to $\mathcal{O}(\log\log N)$,
when the model can be incrementally
maintained. It is worth noting that these two algorithms are agnostic to the context wherein the SV is computed; hence, they are also useful for the applications beyond data valuation. 

{\bf Practical Contribution} Despite the improvements from a theoretical perspective, retraining models for multiple times may still be unaffordable for large datasets and ML models. We then 
introduce two practical SV estimation algorithms specific to ML tasks
by introducing various assumptions
on the utility function. We show that if a learning algorithm is uniformly stable~\cite{bousquet2002stability}, then uniform value division produces a fairly good approximation to the true SV. In addition, for an ML model with smooth loss functions, we propose to use the influence function~\cite{koh2017understanding} to accelerate the data valuation process. However, the efficiency does not come for free. The first algorithm relies on the stability of a learning algorithm, which is difficult to prove for complex ML models, such as deep neural networks. The compromise that we have to make in the second algorithm is that the resulting SV estimates no longer have provable guarantees on the approximation error. Filling the gap between theoretical soundness and practicality is important future work. 

Table~\ref{fig:summary_of_results} summarizes the contributions of this paper. In the rest of the paper, we will elaborate on the idea and analysis of these algorithms, and further use them to compute the data values for various benchmark datasets.

\section{Related Work}

Originated from game theory, the SV, in its
most general form, can be $\mathsf{\#P}$-complete to compute~\cite{deng1994complexity}. Efficiently estimating SV has been studied extensively for decades. For bounded utility functions, Maleki et al.~\cite{maleki2013bounding} described a sampling-based approach that requires $\mathcal{O}(N\log N)$ samples to achieve a desired approximation error in $l_\infty$ norm and $\mathcal{O}(N^2\log N)$ in $l_2$ norm. Bachrach et al.~\cite{bachrach2008approximating} also leveraged a similar approach but focused on the case where the utility function has binary outputs. By taking into account special properties of the utility function, one can derive more efficient approximation algorithms. For instance, Fatima et al.~\cite{fatima2008linear} proposed a probabilistic approximation algorithm with $\mathcal{O}(N)$ complexity for weighted voting games. The game-theoretic analysis of the value of personal data has been explored in~\cite{chessa2017cooperative,kleinberg2001value}, which proposed a fair compensation mechanism based on the SV like ours. They derived the SV under simple data utility models abstracted from network games or recommendation systems, while our work focuses on more complex utility functions derived from ML applications. In our case, the SV no longer has closed-form expressions. We develop novel and efficient approximation algorithms to overcome this hurdle.

Using the SV in the
context of ML is not new. For instance, the SV has been applied to feature selection~\cite{cohen2005feature,sun2012using,mokdad2015determination,sasikala2015novel,lundberg2017unified}.
While their contributions have inspired this paper, many assumptions made for
feature ``valuation'' do not hold for data valuation.
As we will see, by studying the SV
tailored to data valuation, we can develop novel
algorithms that are more efficient than the previous approaches~\cite{maleki2013bounding}.

Despite not being used for data valuation, ranking the importance of training data points has been used for understanding model behaviors, detecting dataset errors, etc. Existing methods include using the influence function~\cite{koh2017understanding} for smooth parametric models and a variant~\cite{sharchilev2018finding} for non-parametric ones. Ogawa et al.~\cite{ogawa2013safe} proposed rules to identify and remove the least influential data in order to reduce the computation cost when training support vector machines (SVM). One can also construct coresets---weighted data subsets---such that models trained on these coresets are provably competitive with models trained on the full dataset~\cite{dasgupta2009sampling}.
These approaches could potentially be used 
for valuing data; however, it is not clear
whether they satisfy the properties desired by data valuation, such as fairness. 
We leave it for future work to understand these distinct approaches for data valuation.

\section{Problem Formulation}
\label{sec:problem_formulation}
Consider a dataset $D=\{ z_i\}_{i=1}^N$ containing data from $N$ users. Let $U(S)$ be the utility function, representing the value calculated by the additive aggregation of $\{z_i\}_{i\in S}$ and $S\subseteq I=\{1,\cdots,N\}$. Without loss of generality, we assume throughout that $U(\emptyset)= 0$. Our goal is to partition $U_{\text{tot}}\triangleq U(I)$, the utility of the entire dataset, to the individual users; more formally, we want to find a function that assigns to user $i$ a number $s(U,i)$ for a given utility function $U$. We suppress the dependency on $U$ when the utility is self-evident and use $s_i$ to represent the value allocated to user $i$. 

The SV~\cite{shapley1953value} is a classic concept in cooperative game theory to attribute the total gains generated by the coalition of all players. Given a utility function $U(\cdot)$, the SV for user $i$ is defined as the average marginal contribution of $z_i$ to all possible subsets of $D=\{z_i\}_{i\in I}$ formed by other users:
\begin{align}
\label{eqn:shapley_definition_no_order}
s_i = \sum_{S\subseteq I\setminus\{i\}} \frac{1}{N{N-1 \choose |S|}}
\big[U(S\cup \{i\})-U(S)\big]
\end{align}
The formula in (\ref{eqn:shapley_definition_no_order}) can also be stated in the equivalent form: 
\begin{align}
\label{eqn:shapley_definition_order}
    s_i = 
\frac{1}{N!}\sum_{\pi \in \Pi(D)}\big[ U(P_i^\pi\cup \{i\}) - U(P_i^\pi)\big]
\end{align}
where $\pi \in \Pi(D)$ is a permutation of users and $P_i^\pi$ is the set of users which precede user $i$ in $\pi$. 
Intuitively, imagine all users' data are to be collected in a random order, and that every user $i$ receives his data's marginal contribution that would bring to those whose data are already collected. If we average these contributions over all the possible orders of users, we obtain $s_i$. The importance of the SV stems from the fact that it is the \emph{unique} value division scheme that satisfies the following desirable properties.

     {\bf 1. Group Rationality}: The value of the entire dataset is completely distributed among all users, i.e., $U(I) = \sum_{i\in I} s_i$.
    
     {\bf 2. Fairness}: (1) Two users who are identical with respect to what they contribute to a dataset's utility should have the same value. That is, if user $i$ and $j$ are equivalent in the sense that $U(S\cup \{i\}) = U(S\cup \{j\}),\forall S\subseteq I\setminus \{i,j\}$, then $s_i=s_j$. (2) Users with zero marginal contributions to all subsets of the dataset receive zero payoff, i.e., $s_i=0$ if $U(S\cup \{i\})=0$ for all $S\subseteq I\setminus\{i\}$.

     {\bf 3. Additivity}: The values under multiple utilities sum up to the value under a utility that is the sum of all these utilities: $s(U,i) + s(V,i) = s(U+V,i)$ for $i\in I$.

The \emph{group rationality} property states that any rational group of users would expect to distribute the full yield of their coalition. The \emph{fairness} property requires that the names of the users play no role in determining the value, which should be sensitive only to how the utility function responds to the presence of a user's data. The \emph{additivity} property facilitates efficient value calculation when data is used for multiple applications, each of which is associated with a specific utility function. With additivity, one can decompose a given utility function into an arbitrary sum of utility functions and compute utility shares separately, resulting in transparency and decentralizability. The fact that the SV uniquely possesses these properties, combined with its flexibility to support different utility functions, leads us to employ the SV to attribute the total gains generated from a dataset to each user.

\section{Efficient SV Estimation}
The challenge in adopting the SV lies in its computational cost. Evaluating the exact SV using Eq.~(\ref{eqn:shapley_definition_no_order}) involves computing the marginal utility of every user to every coalition, which is $\mathcal{O}(2^N)$. Even worse, in many ML tasks, evaluating utility \textit{per se} (e.g., testing accuracy) is computationally expensive as it requires training an ML model. 
In this section, we present various efficient algorithms for approximating the SV. We say that $\hat{s}\in \R^N$ is a $(\epsilon,\delta)$-approximation to the true SV $s=[s_1,\cdots,s_N]^T\in \R^N$ with respect to $l_p$-norm if $P_{\hat{s}}[||\hat{s}_i-s_i||_p\leq \epsilon]\geq 1-\delta$. Throughout this paper, we will measure the approximation error in terms of $l_2$ norm.

\subsection{Baseline: Permutation Sampling}

We start by describing a baseline algorithm~\cite{maleki2015addressing} that approximates the SV for any bounded utility functions with provable guarantees. Let $\pi$ be a random permutation of $I$ and each permutation has a probability of $1/N!$. Consider the random variable $\phi_i = U(P_i^{\pi}\cup\{i\}) - U(P_i^{\pi})$. According to (\ref{eqn:shapley_definition_order}), $s_i=\E[\phi_i]$. Thus, we can estimate $s_i$ by the sample mean. An application of Hoeffding's bound indicates that the number of permutations needed to achieve an $(\epsilon,\delta)$-approximation is $m_{\text{perm}}=(2r^2N/\epsilon^2)\log (2N/\delta)$, where $r$ is the range of the utility function. For each permutation, the utility function is evaluated $N$ times in order to compute the marginal contribution for all $N$ users; therefore, the number of utility evaluations involved in the baseline approach is $m_{\text{eval}}=Nm_{\text{perm}}= \mathcal{O}(N^2\log N)$.

Note that for an ML task, we can write the utility function $U(S)=U_m(A(S))$, where $A(\cdot)$ represents a learning algorithm that maps a dataset $S$ onto a model and $U_m(\cdot)$ is some measure of model performance, such as test accuracy. Typically, a substantial part of computational costs associated with the utility evaluation lies in $A(\cdot)$. Hence, it is useful to examine the efficiency of an approximation algorithm in terms of the number of model training required. In general, one utility evaluation would need to re-train a model. Particularly, when $A(\cdot)$ is incrementally trainable, one pass over the entire training set allows us to evaluate $\phi_i$ for all $i=1,\cdots,N$. Hence, in this case, the number of model training needed to achieve an $(\epsilon,\delta)$-approximation is the same as $m_\text{perm}=\mathcal{O}(N\log N)$. 

\subsection{Group Testing-Based Approach}
We now describe an algorithm that makes the same assumption of bounded utility as the baseline algorithm, but requires significantly fewer utility evaluations than the baseline.

Our proposed approximation algorithm is inspired by previous work applying the group testing theory to feature selection~\cite{zhou2014parallel}. Recall that group testing is a combinatorial search paradigm~\cite{du2000combinatorial}, in which one wants to determine whether each item in a set is ``good'' or ``defective'' by performing a sequence of tests. The result of a test may be positive, indicating that at least one of the items of that subset is defective, or negative, indicating that all items in that subset are good. Each test is performed on a pool of different items and the number of tests can be made significantly smaller than the number of items by smartly distributing items into pools. Hence, the group testing is particularly useful when testing an individual item's quality is expensive. Analogously, we can think of SV calculation as a group testing problem with continuous quality measure. Each user's data is an ``item'' and the data utility corresponds to the item's quality. Each ``test'' in our scenario corresponds to evaluating the utility of a subset of
users and is expensive. Drawing on the idea of group testing, we hope to recover the utility of all user subsets from a small amount of customized tests.

Let $T$ be the total number of tests. At test $t$, a random set of users is drawn from $I$ and we evaluate the utility of the selected set of users. If we model the appearance of user $i$ and $j$'s data in a test as Boolean random variables $\beta_i$ and $\beta_j$, respectively, then the difference between the utility of user $i$ and that of user $j$ is
\begin{align}
\label{eqn:expected_utility_diff}
(\beta_i-\beta_j)U(\beta_1,\cdots,\beta_N)
\end{align}
where $U(\beta_1,\cdots,\beta_N)$ is the utility evaluated on the users with the Boolean appearance random variable equal to $1$. 

Using the definition of the SV, one can derive the following formula of the SV difference between any pair of users.
\begin{Lem}
\label{lm:shapley_diff}
For any $i,j\in I$, the difference in SVs between $i$ and $j$ is 
\begin{align}
\label{eqn:shapley_diff}
    s_i-s_j = \frac{1}{N-1} \!\! \sum_{S\subseteq I\setminus\{i,j\}} \!\!\!\!\frac{U(S\cup\{i\}) - U(S\cup \{j\})}{\binom{N-2}{|S|}}
\end{align}
\end{Lem}
Due to the space limitation, we omit all the proofs of the paper to our supplemental materials. The key idea of the proposed algorithm is to smartly design the sampling distribution of $\beta_1,\cdots,\beta_N$ such that the expectation of (\ref{eqn:expected_utility_diff}) mirrors the Shapley difference in (\ref{eqn:shapley_diff}). This will enable us to calculate the Shapely differences from the test results with a high-probability error bound. The following Lemma states that if we can estimate the Shapley differences between all data pairs up to $(\epsilon/\sqrt{N},\delta/N)$, then we will be able to recover the SV with the approximation error $(\epsilon,\delta)$.

\begin{Lem}
\label{lm:feasibility}
Suppose that $C_{ij}$ is an $(\epsilon/(2\sqrt{N}),\delta/(N(N-1)))$-approximation to $s_i-s_j$. Then, any solutions to the feasibility problem 
    \begin{align}
    \label{eqn:feasible_1}
        &\sum_{i=1}^N\hat{s}_i = U_\text{tot}\\
        \label{eqn:feasible_2}
        &|(\hat{s}_i-\hat{s}_j)-C_{i,j}|\leq \epsilon/(2\sqrt{N}) \quad \forall i,j\in \{1,\ldots,N\}
    \end{align}
is an $(\epsilon,\delta)$-approximation to $s$ with respect to $l_2$-norm.
\end{Lem}

Algorithm~\ref{alg:gt} presents the pseudo-code of the group testing-based algorithm, which first estimates the Shapley differences and then derives the SV from the Shapley differences by solving a feasibility problem.

\begin{algorithm}[ht]
\SetAlgoLined
\SetKwInOut{Input}{input}
\SetKwInOut{Output}{output}
\Input{Training set - $D = \{(x_i,y_i)\}_{i=1}^N$, utility function $U(\cdot)$, the number of tests - $T$}
\Output{The estimated SV of each training point - $\hat{s}\in \R^N$}

$Z \leftarrow 2\sum _{k=1}^{N-1} \frac{1}{k}$\;

$q(k)\leftarrow\frac{1}{Z} (\frac{1}{k} + \frac{1}{N-k})$ for $k=1,\cdots,N-1$\;

Initialize $\beta_{ti} \leftarrow 0$, $t= 1, ..., T, i = 1,..., N $\;
\For{$t = 1$ to $T$}{
Draw $k_t \sim q(k)$\;
         Uniformly sample a length-$k_t$ sequence $S$ from $\{1,\cdots,N\}$ \;
        
        $\beta_{ti} \leftarrow 1$ for all $i\in S$\;
    $u_t \leftarrow U(S)$\;
}
$\Delta U_{ij} \leftarrow \frac{Z}{T} \sum_{t=1}^T u_t (\beta_{ti}- \beta_{tj})$ for $i = 1,..,N$, $j = 1,...,N$ and $j\geq i$  \;

Find $\hat{s}$ by solving the feasibility problem $\sum_{i=1}^N\hat{s}_i = U(I), |(\hat{s}_i-\hat{s}_j)-\Delta U_{i,j}|\leq \epsilon/(2\sqrt{N}), \forall i,j\in \{1,\cdots,N\}$\;

 \caption{Group Testing Based SV Estimation.}
 \label{alg:gt}
\end{algorithm}

The following theorem provides a lower bound on the number of tests $T$ needed to achieve an
$(\epsilon, \delta)$-approximation.

\begin{Thm}
\label{thm:gt_T}
Algorithm~\ref{alg:gt} returns an $(\epsilon,\delta)$-approximation to the SV with respect to $l_2$-norm if the number of tests $T$ satisfies $T\geq 8 \log \frac{N(N-1)}{2\delta} /\big((1-q_{tot}^2)h\big(\frac{\epsilon}{Zr\sqrt{N}(1-q_{tot}^2)}\big)\big) $, where $q_{tot}=\frac{N-2}{N}q(1) + \sum_{k=2}^{N-1} q(k)[1+\frac{2k(k-N)}{N(N-1)}]$, $h(u) =(1+u)\log(1+u) - u$,
$Z=2\sum _{k=1}^{N-1} \frac{1}{k}$, and $r$ is the range of the utility function.\footnote{Wang and Jia \cite{wang2023note} gives an improved version of this result.}
\end{Thm}

Using the Taylor expansion of $h$, it can be proved that when $N$ is large, $T$ is $\mathcal{O}(N(\log N)^2)$.
Since only one utility evaluation is required for a single test, the number of utility evaluations is at most $\mathcal{O}(N(\log N)^2)$.
On the other hand, in the baseline approach,
the number of utility evaluations is $\mathcal{O}(N^2\log N)$. Hence, the group testing requires significantly fewer model evaluations than the baseline.

\subsection{Exploiting the Sparsity of Values}
We now present an algorithm inspired by our empirical observations of the SV for large datasets. This algorithm can produce an $(\epsilon,\delta)$-approximation to the SV with only $\mathcal{O}(N\log(N)\log(\log(N)))$ utility evaluations.

Figure~\ref{fig:shapley_dist} illustrates the distribution of the SV of the MNIST dataset, from which we observed that the SV is ``approximately sparse''---most of values are concentrated around its mean and only a few data points have significant values. In the literature, the ``approximate sparsity'' of a vector $s$ is characterized by a small error of its best $K$-term approximation:
\begin{align}
    \sigma_K(s) = \inf\{\|s-z\|_1, \text{ z is $K$-sparse}\}
\end{align}
This observation opens up a vast collection of tools from compressive sensing for the purpose of calculating the SV.

\begin{figure}[t]
\centering
\includegraphics[width=0.9\columnwidth]{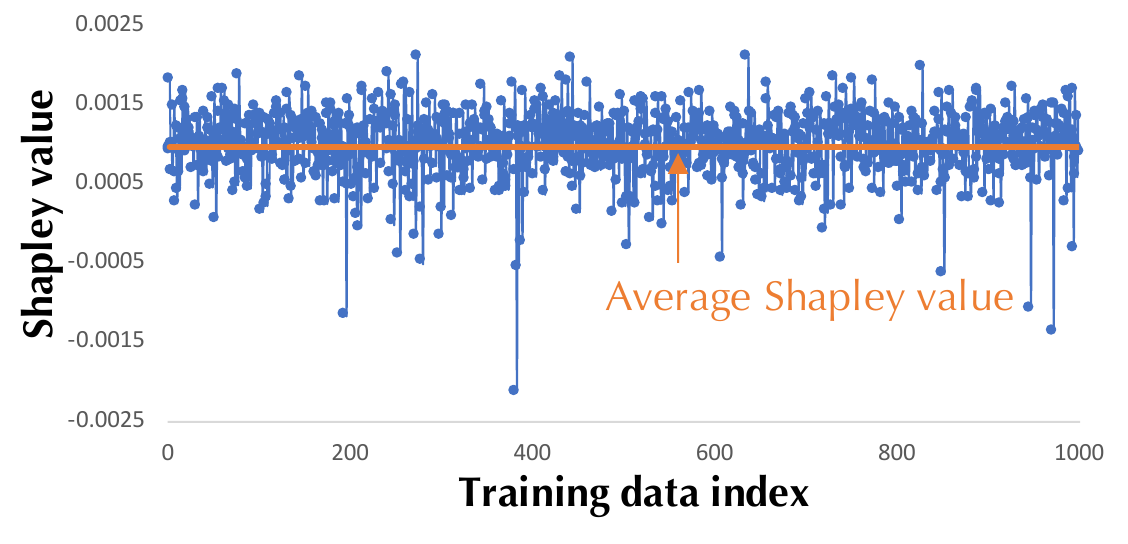}
\caption{The distribution of the SV of a size-$1000$ training set randomly sampled from MNIST. $\sigma_{367}(s)/(\sum_{i=1}^Ns_i)=0.5$. The utility function is the test accuracy. }
\label{fig:shapley_dist}
\end{figure}

Compressive sensing studies the problem of recovering a sparse signal $s$ with far fewer measurements $y=As$ than the length of the signal. A sufficient condition for recovery is that the measurement matrix $A\in \R^{M\times N}$ satisfies a key property, the \emph{Restricted Isometry Property (RIP)}. In order to ensure that $A$ satisfies this property, we simply choose $A$ to be a random Bernoulli matrix. The results in random matrix theory imply that $A$ satisfies RIP with high probability. Define the $k$th restricted isometry constant $\delta_k$ for a matrix $A$ as 
\begin{align}
    \delta_k(A) = \min&\{\delta:\forall s, 
   \|s\|_0\leq k,  \nonumber\\
   &(1-\delta)\|s\|_2^2 \leq \|As\|_2^2 \leq (1+\delta)\|s\|_2^2 
\end{align}

It has been shown in~\cite{rauhut2010compressive} that every $k$-sparse vector $s$ can be recovered by solving a convex optimization problem 
\begin{align}
    \min_{s\in \R^N }\|s\|_1, \quad \text{s.t. } As = y
\end{align}
if $\delta_{2s}(A)<1/3$. This result can also be generalized to noisy measurements~\cite{candes2006stable}. Drawing on the ideas of compressed sensing, we present Algorithm~\ref{alg:compressive_perm}, termed compressive permutation sampling.

\begin{algorithm}[ht]
\SetAlgoLined
\SetKwInOut{Input}{input}
\SetKwInOut{Output}{output}
\Input{Training set - $D = \{(x_i,y_i)\}_{i=1}^N$, utility function $U(\cdot)$, the number of measurements - $M$, the number of permutations - $T$}
\Output{The SV of each training point - $\hat{s}\in \R^N$}
 Sample a Bernoulli matrix $A$, where $A_{m,i}\in \{-1/\sqrt{M},1/\sqrt{M}\}$ with equal probability\;

    \For{$t\gets 1$ \KwTo $T$}{
        $\pi_t \leftarrow \text{GenerateUniformRandomPermutation}(D)$\;
        
        $\phi^t_i \leftarrow  U(P_i^{\pi_t}\cup \{i\}) -  U(P_i^{\pi_t})$ for $i=1,\ldots,N$\;
        
         \For{$m\gets 1$ \KwTo $M$}{
        $\hat{y}_{m,t} \leftarrow \sum_{i=1}^N A_{m,i}\phi^t_i$\;
        }
    }
    $\bar{y}_m = \frac{1}{T}\sum_{t=1}^T \hat{y}_{m,t}$ for $m=1,\ldots,M$\;
    $\bar{s} = U(D)/N$\;
    $\Delta s^*\leftarrow  \argmin_{\Delta s\in\R^N} \|\Delta s\|_1, \text{s.t. } \|A(\bar{s}+\Delta s)-\bar{y}\|_2\leq \epsilon$\;
    $\hat{s} = \bar{s} + \Delta s^*$\;
 \caption{Compressive Permutation Sampling.}
 \label{alg:compressive_perm}
\end{algorithm}

\begin{Thm}
\label{thm:compress_perm}
Suppose that $U(\cdot)$ is monotone. There exists some constant $C'$ such that if $M\geq C'(K\log(N/(2K))+\log(2/\delta))$ and $T\geq \frac{2r^2}{\epsilon^2} \log \frac{4M}{\delta}$, except for an event of probability no more than $\delta$, the output of Algorithm~\ref{alg:compressive_perm} obeys
\begin{align}
        \|\hat{s}-s\|_2 \leq C_{1,K}\epsilon + C_{2,K} \frac{\sigma_K(s)}{\sqrt{K}}
\end{align}
for some constants $C_{1,K}$ and $C_{2,K}$.
\end{Thm}

Therefore, the number of utility evaluations (and model training) required for achieving the approximation error guarantee in Theorem~\ref{thm:compress_perm} is $NT=\mathcal{O}(N\log(\log(N)))$. Particularly, when the utility function is defined with respect to an incrementally trainable model, only $\log\log(N)$ full model training is needed for achieving the error guarantee. 


\subsection{Stable Learning Algorithms}
\label{sec:stability}
A learning algorithm is {\em stable} if the model learned by the algorithm is insensitive to the removal of an arbitrary point in the training dataset~\cite{bousquet2002stability}. More specifically, an algorithm $G$ has uniform stability $\gamma$ with respect to the loss function $l$ if $\|l(G(S),\cdot) - l(G(S^{\setminus i}),\cdot)\|_\infty \leq \gamma$ for all $i\in \{1,\cdots,|S|\}$, where $S$ denotes the training set and $S^{\setminus i}$ denotes the one by removing $i$th element of $S$. Indeed, a broad variety of learning algorithms are stable, including all learning algorithms with Tikhonov regularization. Stable learning algorithms are appealing as they enjoy provable generalization error bounds~\cite{bousquet2002stability}. Assume that the model is trained via a stable learning algorithm and training data's utility is measured in terms of the testing loss. Due to the inherent insensitivity of a stable learning algorithm to the training data, we expect that the SV of each training point is similar to one another. The following theorem confirms our intuition and provides an upper bound on the SV difference between any pair of training data points.

\begin{Thm}
\label{thm:stable}
For a learning algorithm $A(\cdot)$ with uniform stability $\beta = \frac{C_\text{stab}}{|S|}$, where $|S|$ is the size of the training set
and $C_\text{stab}$ is some constant. Let the utility of $D$ be $U(D) = M - L_\text{test}(A(D),D_\text{test})$, where $L_\text{test}(A(D),D_\text{test}) = \frac{1}{N}\sum_{i=1}^N l(A(D),z_{\text{test},i})$ and $0\leq l(\cdot,\cdot)\leq M$. Then,
$s_i-s_j\leq 2C_\text{stab}\frac{1+\log(N-1)}{N-1}$ and the Shapley difference vanishes as $N\rightarrow \infty$. 
\end{Thm}

By Lemma~\ref{lm:feasibility}, if $2C_\text{stab}\frac{1+\log(N-1)}{N-1}$ is less than $\epsilon/(2\sqrt{N})$, uniformly assigning $\frac{U_{\text{tot}}}{N}$ to each data contributor provides an $(\epsilon,0)$-approximation to the SV.

\subsection{Heuristic Based on Influence Functions}
\label{sec:influence}
Computing the SV involves evaluating the change in utility of all possible sets of data points after adding one more point. A plain way to evaluate the difference requires training a large number of models on different subsets of data. Koh et al.~\cite{koh2017understanding} show that influence functions can be used as an efficient approximation of parameter changes after adding or removing one point. Therefore, the need for re-training models is circumvented. Assume that model parameters are obtained by solving an empirical risk minimization problem $\hat{\theta}^m= \argmin_{\theta} \frac{1}{m}\sum_{i=1}^m l(z_i,\theta)$. Applying the result in~\cite{koh2017understanding}, we can approximate the parameters learned after adding $z$ by using the relation $\hat{\theta}^{m+1}_{z} =  \hat{\theta}^m -\frac{1}{m} H_{\hat{\theta}^m}^{-1} \nabla_\theta L(z,\hat{\theta}^m)$
where $H_{\hat{\theta}^m}=\frac{1}{m}\sum_{i=1}^m \nabla_\theta^2 L(z_i,\hat{\theta}^m)$ is the Hessian. The parameter change after removing $z$ can be approximated similarly, except for replacing the $-$ by $+$ in the above formula.
The efficiency of the baseline permutation sampling method can be significantly improved by combining it with influence functions. Moreover, we can employ a more sophisticated sampling scheme to reduce the variance of the result. Indeed, we can re-write the SV as $s_i = \frac{1}{N}\sum_{k=1}^N \E[X_i^k]$, where $X_i^k = U(S\cup \{i\})-U(S)$ is the marginal contribution of user $i$ to a size-$k$ subset that is randomly selected
with probability $1/{N-1\choose k}$. This suggests that stratified sampling can be used to approximate the SV, which customizes the number of samples for estimating each expectation term according to the variance of $X_i^k$.

{\bf Largest-$S$ Approximation.\quad} One practical
heuristic of using influence functions is to consider a single subset $S$ for computing $s_i$, namely, $I \setminus \{i\}$.   
With this heuristic, we can simply take a trained
model on the whole dataset, and calculate the influence
function for each data point. For logistic regression models, the first and second derivations enjoy closed-form expressions and the change in parameters after removing one point $z=(x,y)$ can be approximated by $-\big(\sum_{i=1}^N \sigma(x_i^T\hat{\theta}^{N})\sigma(-x_i^T\hat{\theta}^{N})x_ix_i^T\big)^{-1} \sigma(-y x_i^T\hat{\theta}^{N})yx$ where
$\sigma(u) = 1/(1+\exp(-u))$ and $y\in \{-1,1\}$. The fact that largest-$S$ influence
only considers a single subset makes it impossible to satisfy the {\em group rationality} and {\em additivity} properties simultaneously.
\begin{Thm}
\label{thm:inf_violation}
Consider the value attribution scheme that assigns the value $\hat{s}(U,i) = C_U[U(S\cup \{i\})-U(S)]$ to user $i$ where $|S|=N-1$ and $C_U$ is a constant such that $\sum_{i=1}^N \hat{s}(U,i) = U(I)$. Consider two utility functions $U(\cdot)$ and $V(\cdot)$. Then, $\hat{s}(U+V,i)\neq\hat{s}(U,i) + \hat{s}(V,i)$ unless $V(I)[\sum_{i=1}^N U(S\cup\{i\})-U(S)] = U(I) [\sum_{i=1}^N V(S\cup\{i\})-V(S)]$.
\end{Thm}


\section{Experimental Results}

\paragraph{Comparing Approximation Accuracy.\quad} We first compare the proposed approximation methods that only require mild assumptions on the ML models (e.g., bounded or differentiable utility), including (a) the permutation sampling baseline, (b) the group testing-based method, (c) using influence functions to approximate all marginal contributions, and (d) approximating the SV with only the influence function to the largest subset. The last two methods are hereinafter referred to as \emph{all-$S$ influence} and \emph{largest-$S$ influence}, respectively. We use a small-scale dataset, \texttt{iris}, 
and use (a) to estimate the true
SV for a regularized logistic regression up to $\epsilon=1/N$. Figure~\ref{fig:adv}(d) shows
that the approximations produced by (a)-(c) are closest to each other. The result of the largest-$S$ influence is correlated with that of the other techniques, although it cannot recover the true SV.

\vspace{-0.2cm}
\paragraph{Runtime comparison.} We implement the SV calculation techniques on a machine with 16 cores (Intel Xeon CPU E5-2620 v4 @ 2.10GHz) and compare the runtime of different techniques on a two-class \texttt{dog-vs-fish} dataset~\cite{koh2017understanding} of size $900$ constructed from the ImageNet dataset. To evaluate the runtime for training sizes above $900$, we concatenate duplicate copies of the dog-vs-fish dataset. For each training data point, we first pre-compute the 2048-dimensional inception features and then train a logistic regression using the stochastic gradient descent for $150$ epochs. The utility function is the negative testing loss of the logistic regression model. For the largest-$S$ influence and the all-$S$ influence, we use the method in~\cite{koh2017understanding} to compute the influence function. The runtime of different techniques in logarithmic scale is displayed in Figure~\ref{fig:comp} (b). We can see that the group testing-based method outperforms the permutation sampling baseline by several orders of magnitude for a large number of data points. By exploiting influence function heuristics and the stratified sampling trick in Section~\ref{sec:influence}, the computational costs can be further reduced. Due to the fact that the largest-$S$ influence heuristic only focuses on the marginal contribution of each training data point to a single subset, it is much more efficient than the permutation sampling, group testing and the all-$S$ influence, which compute the marginal contributions to a large number of subsets.

\begin{figure}[t]
\centering
\includegraphics[width=1.0\columnwidth]{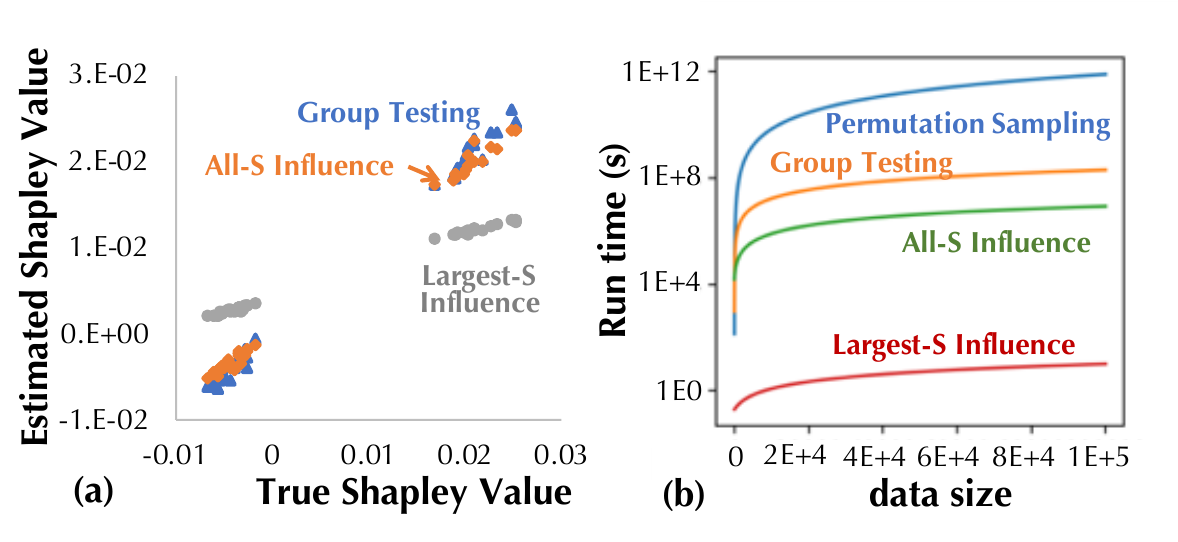}
\caption{Consider the SV approximation methods that do not rely on specific assumptions on the underlying learning algorithms and compare the (a) data values produced by them for training a logistic regression model and (b) their runtime.}
\label{fig:comp}
\end{figure}

\vspace{-0.2cm}
\paragraph{Approximation under sparsity assumptions.} When it is plausible to assume the SV of a training set is sparse, we could employ the idea of compressive sensing to recover the SV with fewer samples. Figure~\ref{fig:comp_sparse} compares the sample efficiency of the baseline permutation sampling and the compressive permutation sampling method on a size-$1000$ dataset sampled randomly from \texttt{MNIST}. For a given approximation error, the compressive permutation requires significantly fewer samples and model valuations than the baseline approach. The superiority of the compressive permutation becomes less evident at the large sample regime.

\begin{figure}[t]
\centering
\includegraphics[width=0.5\columnwidth]{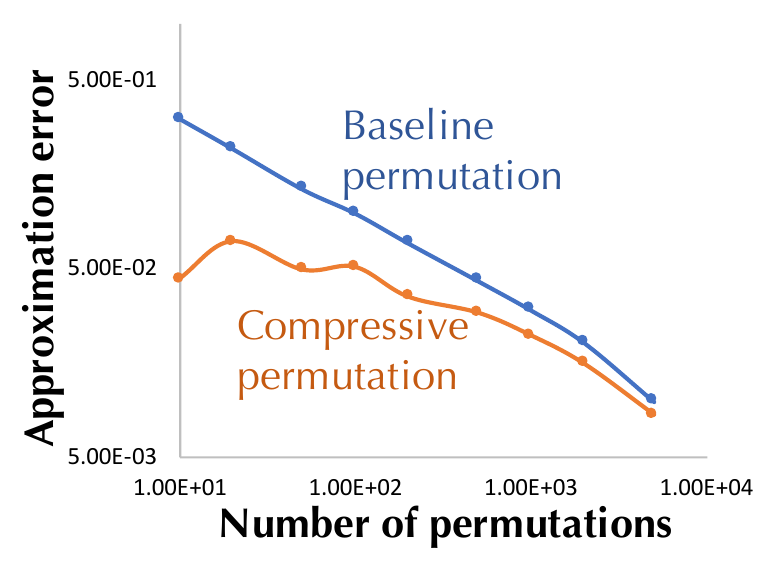}
\caption{Comparison of approximation errors with different numbers of permutations for the baseline permutation sampling and the compressive permutation sampling method.}
\label{fig:comp_sparse}
\end{figure}

\vspace{-0.2cm}
\paragraph{Stable learning algorithms.}\quad Our theoretical result in Section~\ref{sec:stability} shows that the SV of training data tends to be uniform for a stable learning algorithm, which has a small stability parameter $\beta$. We empirically validate this result by training a ridge regression on the \texttt{diabetes} dataset and varying the strength of its regularization term. In~\cite{bousquet2002stability}, it is shown that the stability parameter $\beta$ of the ridge regression $\min_\theta \frac{1}{N}\sum_{i=1}^N l(\theta, z_i) + \lambda \|\theta\|^2$ is proportional to $\sigma^2/\lambda$, where $\sigma$ is the Lipschitz constant of the loss function with respect to the model parameter $\theta$ and equal to $2|x_i^T\theta-y_i|\cdot|x_i|$. When the model fits the training data well, the change in $\sigma$ is small; therefore, applying more regularization leads to a more stable learning algorithm, which has lower variance in the training data values as illustrated in the shaded area of Figure~\ref{fig:stable_noise}. On the other hand, if the model no longer fits the data well due to excessive regularization, then $\sigma$ will dominate the stability parameter. In this case, since $\sigma$ increases with the regularization strength, $\beta$ and thereby the variance of the SV also increase. Note that the variance of the SV is identical to the approximation error of a uniform value division scheme.

\begin{figure}[t]
\centering
\includegraphics[width=\columnwidth]{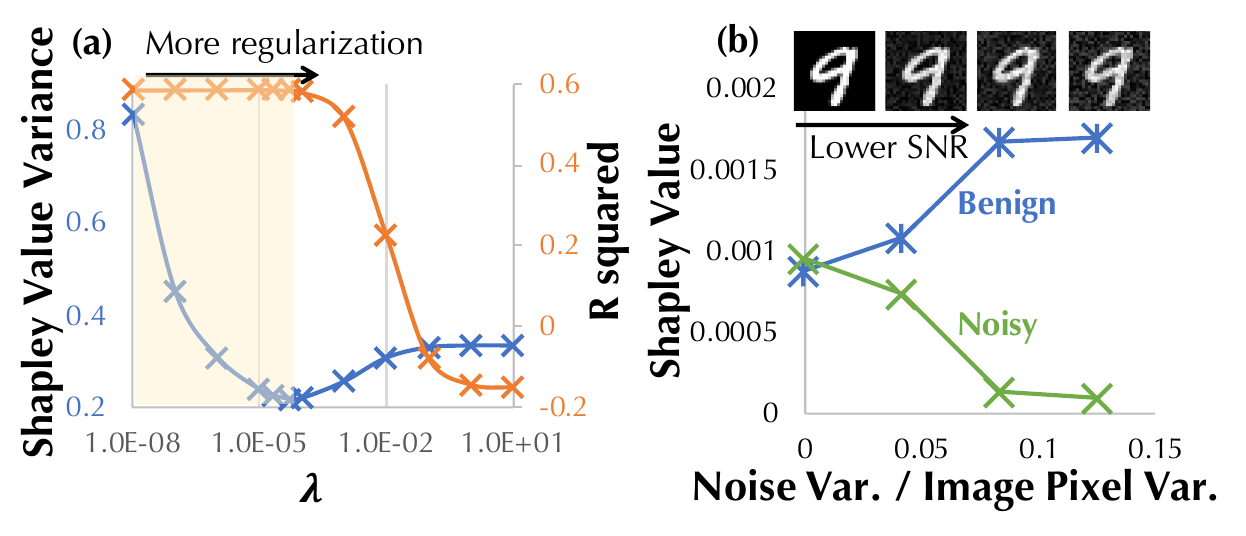}
\caption{(a) Variance of data values for a ridge regression with different regularization strength ($\lambda$). (b) Tradeoff between data value and privacy.}
\label{fig:stable_noise}
\end{figure}

\vspace{-0.2cm}
\paragraph{Value for Privacy-Preserving Data.}\quad
Differential privacy~\cite{dwork2008differential} has emerged as a standard privacy notation and is often achieved by adding noise that has a magnitude proportional to the desired privacy level. On the other hand, noise diminishes the usefulness of data and thereby degrades the value of data. We construct a training set using the \texttt{MNIST}, and divide the training dataset into two halves, one half containing normal images and the other half containing noisy ones. The testing accuracy on normal images is used as the utility function. Figure~\ref{fig:stable_noise}(b) illustrates a clear tradeoff between privacy and data value - the SV decreases as data becomes noisier.

\vspace{-0.2cm}
\paragraph{Value for Adversarial Examples.\quad} Mixing adversarial examples with benign examples in the training dataset, or adversarial training, is an effective method to improve the adversarial robustness of a model. In practice, we measure the robustness in terms of the testing accuracy on a dataset containing adversarial examples. We expect that the adversarial examples in the training dataset become more valuable as more adversarial examples are added into the testing dataset. Based on the \texttt{MNIST}, we construct a training dataset that contains both benign and adversarial examples and synthesize testing datasets with different adversarial-benign mixing ratios. Two popular attack algorithms, namely, Fast Gradient Sign Method (FGSM)~\cite{goodfellow2014explaining} and the Carlini and Wagner (CW) attack~\cite{carlini2017towards} are used to generate adversarial examples. Figure~\ref{fig:adv}(a, b) compares the average SV for adversarial examples and for benign examples in the training dataset. The negative testing loss for logistic regression is used as the utility function. We see that the SV of adversarial examples increases as the testing data becomes more adversarial and contrariwise for benign examples. This is consistent with our expectation. In addition, the adversarial examples in the training set are more valuable if they are generated from the same attack algorithm for testing adversarial examples.

\begin{figure}[t]
\centering
\includegraphics[width=1.0\columnwidth]{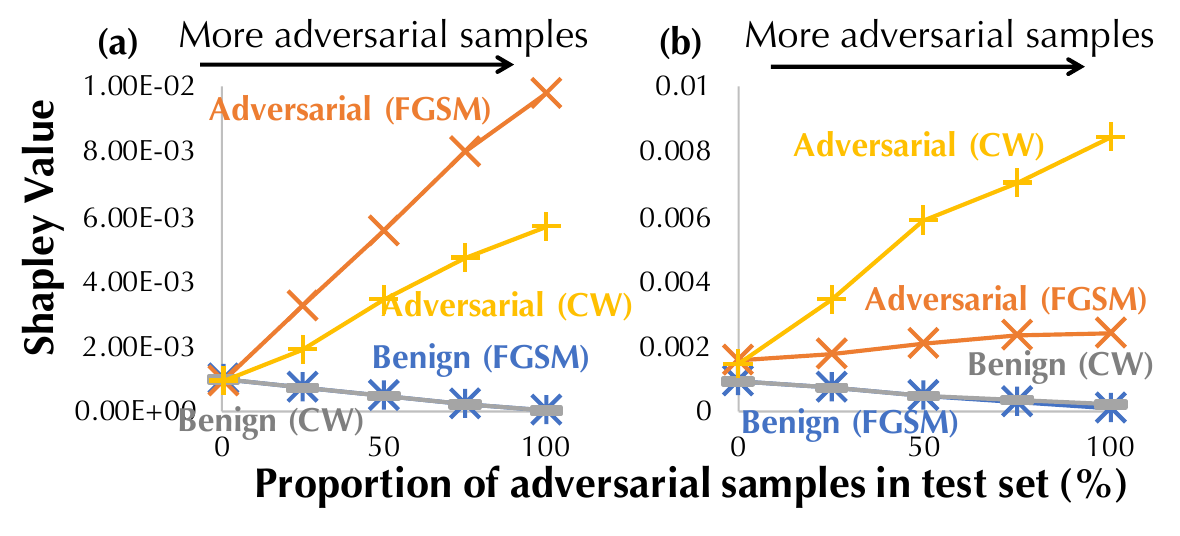}
\caption{(a, b) Comparison of SV of benign and adversarial examples. FGSM and CW are different attack algorithms used for generating adversarial examples in the testing dataset: (a) (resp. (b)) is trained on Benign + FGSM (resp. CW) adversarial examples.}
\label{fig:adv}
\end{figure}

\vspace{-0.2cm}

\section{Conclusion}

ML has opened up exciting opportunities to tackle a wide variety of problems; nevertheless, very few works have attempted to understand the value of data used for training models. A principled way of data valuation is the key to stimulating data exchange, enabling the development of more sophisticated and robust ML models. We adopt the SV, a classic concept from cooperative game theory, for data valuation. The SV has many unique properties appealing to data valuation. However, the lack of efficient methods to compute the SV has prevented it from being adopted in the past. We develop a repertoire of techniques for estimating the SV in different scenarios. 

For future work, We wish to continue exploring the connection between ML and game theory and develop efficient valuation methods for ML models. It is also critical to understand other concepts from cooperative game theory (e.g., stable coalition) in the context of data valuation. Last but not least, we hope to apply the techniques to real-world applications and revolutionize the way of data collection and dissemination.

\subsubsection*{Acknowledgements}
We would like to thank Tan Pin Lin and Feng Mingling for helping correct the complexity calculation of the group testing-based approximation algorithm and the assumption for proving the complexity of the compressive permutation sampling algorithm in the earlier version of the paper. We would also like to thank Jiachen T. Wang for enhancing the paper's readability and strengthening some proofs.

This work is supported in part by the Republic of Singapore’s National Research
Foundation through a grant to the Berkeley Education Alliance for Research in
Singapore (BEARS) for the Singapore-Berkeley Building Efficiency and
Sustainability in the Tropics (SinBerBEST) Program. This work is also supported in part by the CLTC (Center for Long-Term
Cybersecurity); FORCES (Foundations Of Resilient
CybEr-Physical Systems), which receives support from
the National Science Foundation (NSF award numbers
CNS-1238959, CNS-1238962, CNS-1239054, CNS1239166);
and the National Science Foundation under
Grant No. TWC-1518899. 
CZ and the DS3Lab gratefully acknowledge the support from Mercedes-Benz Research
\& Development NA, MeteoSwiss, Oracle Labs, Swiss Data Science Center, Swisscom, Zurich Insurance, Chinese Scholarship Council, and the Department of Computer Science at ETH Zurich.



\bibliography{aomsample}
\bibliographystyle{abbrv}

\appendix
\onecolumn
\section{Proof of Lemma 1}
\begin{Lm}
\label{lm:shapley_diff} 
For any $i,j\in I$ and $i\neq j$, the difference in Shapley values between $i$ and $j$ is 
\begin{align}
    s_i-s_j  = \frac{1}{N-1}  \sum_{S\subseteq I\setminus\{i,j\}} \frac{1}{\binom{N-2}{|S|}}\big[U(S\cup\{i\}) - U(S\cup \{j\})\big]\nonumber
\end{align}
\end{Lm}

\begin{proof}
\begin{align}
 &s_i-s_j= \sum_{S\subseteq I\setminus\{i\}} \frac{|S|!(N-|S|-1)!}{N!} \big[U(S\cup \{i\})-U(S)\big] - \sum_{S\subseteq I\setminus\{j\}} \frac{|S|!(N-|S|-1)!}{N!} \big[U(S\cup \{j\})-U(S)\big]  \nonumber \\
 &=\sum_{S\subseteq I\setminus\{i,j\}} \frac{|S|!(N-|S|-1)!}{N!} \big[U(S\cup\{i\}) - U(S\cup \{j\})\big] +\sum_{S\in\{T|T\subseteq I,i\notin T, j\in T\}}  \frac{|S|!(N-|S|-1)!}{N!}   \big[U(S\cup \{i\})-U(S)\big]\nonumber \\
 &-\sum_{S\in\{T|T\subseteq I,i\in T, j\notin T\}} \frac{|S|!(N-|S|-1)!}{N!}\cdot \big[U(S\cup \{j\})-U(S)\big]\nonumber\\
 &=\sum_{S\subseteq I\setminus\{i,j\}} \frac{|S|!(N-|S|-1)!}{N!} \big[U(S\cup\{i\}) - U(S\cup \{j\})\big]\nonumber\\
 &+ \sum_{S'\subseteq I\setminus\{i,j\}} \frac{(|S'|+1)!(N-|S'|-2)!}{N!}\big[U(S'\cup\{i\}) - U(S'\cup \{j\})\big]\nonumber\\
 &= \sum_{S\subseteq I\setminus\{i,j\}} \big( \frac{|S|!(N-|S|-1)!}{N!} +\frac{(|S|+1)!(N-|S|-2)!}{N!}\big) \cdot \big[U(S\cup\{i\}) - U(S\cup \{j\})\big]\nonumber\\
 &=\frac{1}{N-1}  \sum_{S\subseteq I\setminus\{i,j\}} \frac{1}{C_{N-2}^{|S|}}\big[U(S\cup\{i\}) - U(S\cup \{j\})\big] \, .\nonumber
\end{align}
\end{proof}

Loosely speaking, the proof distinguishes subsets $S$ which include neither $i$ nor $j$ (such that the subset utility $U(S)$ of the marginal contribution directly cancels) and subsets including either $i$ or $j$. In the latter case, $S$ can be partitioned to a mock subset $S'$ by excluding the
respective point from S such that a common sum over $S'$ again eliminates all terms other than $U(S' \cup \{i\}) -  U(S' \cup \{j\})$.

\section{Proof of Lemma 2}
\begin{Lm}
\label{lm:feasibility}
Suppose that $C_{ij}$ is an $(\epsilon/(2\sqrt{N}),\delta/(N(N-1)))$-approximation to $s_i-s_j$. Then, the solution to the feasibility problem 
\begin{align}
\label{eqn:feasible_1}
    &\sum_{i=1}^N\hat{s}_i = U_\text{tot}\\
    \label{eqn:feasible_2}
    &|(\hat{s}_i-\hat{s}_j)-C_{i,j}|\leq \epsilon/(2\sqrt{N}) \quad \forall i,j\in \{1,\ldots,N\}
\end{align}
must exist, and any feasible solutions are an $(\epsilon,\delta)$-approximation to $s$ with respect to $l_2$-norm.
\end{Lm}

\begin{proof}
To see the existence of the feasible solution, the true Shapley value $\hat s_i = s_i$ is a feasible solution given the condition.

For the second part of the theorem, let $\epsilon'=\epsilon/(2\sqrt{N})$. Assume, for contradition, $\hat{s}_i-s_i> \epsilon/\sqrt{N}$. Let $\hat{s}_i-s_i= c\epsilon'$ where $c>2$. 
 
 Since $C_{i,j}$ is an $(\epsilon',\delta/(N(N-1)))$-approximation to $s_i-s_j$, we have that with probability at least $1- \delta/(N(N-1))$, 
 \begin{align}
     |(s_i-s_j)-C_{i,j}|\leq \epsilon'
 \end{align}
 
  Moreover, the inequality (\ref{eqn:feasible_2}) implies that 
\begin{align*}
    |(\hat{s}_i-\hat{s}_j)-C_{i,j}| \leq \epsilon'
\end{align*}
Therefore,
\begin{align}
   & |\hat{s}_i - s_i + s_j - \hat{s}_j| =|\hat{s}_i - \hat{s}_j - C_{i,j}- (s_i - s_j -C_{i,j})|\\
    &\leq |\hat{s}_i - \hat{s}_j - C_{i,j}| + |s_i - s_j -C_{i,j}|\\
    &\leq 2\epsilon'
\end{align}
with probability at least $1-\delta/(N(N-1))$. By the assumption that $\hat{s}_i-s_i = c\epsilon'$ and $c>2$, we have
\begin{align}
    (c-2)\epsilon'\leq \hat{s}_j - s_j\leq (c+2) \epsilon'
\end{align}
which further implies that $\hat{s}_j - s_j>0$ for some $j\neq i$. Thus, with probability $1-\delta/N$, we have $\hat{s}_j - s_j>0$ for all $j\neq i$. 

Then,
\begin{align}
    \sum_{j=1}^N (\hat{s}_j - s_j) = \sum_{j\neq i} (\hat{s}_j - s_j) + (\hat{s}_i - s_i) > 0
\end{align}
Since $\sum_{j=1}^N s_j = U_\text{tot}$, it follows that $\sum_{j=1}^N \hat{s}_j> U_\text{tot}$, which contradicts with the fact that $\hat{s}_j$ ($j=1,\ldots,N$) is a solution to the feasibility problem (\ref{eqn:feasible_1}) and (\ref{eqn:feasible_2}). 


The contradiction can be similarly established for $s_i-\hat{s}_i = c\epsilon'$. Therefore, we have that with probability at least $1-\delta/N$, $|s_i-\hat{s}_i|\leq 2\epsilon'$ for some $i$. This in turn implies that with probability at least $1-\delta$, $\|\hat{s}-s\|_\infty \leq 2\epsilon' = \epsilon/\sqrt{N}$. 
Moreover, since $\|\hat{s}-s\|_2\leq \sqrt{N}\|\hat{s}-s\|_\infty = \epsilon$, we have that $\|\hat{s}-s\|_2 \leq \epsilon$ with probability at least $1-\delta$.


\end{proof}

\section{Proof of Theorem 3}
\label{app:theorem_gt}
We prove Theorem 3, which specifies a lower bound on the number of tests needed for achieving a certain approximation error. Before delving into the proof, we first present a lemma that is useful for establishing the bound in Theorem~3.




\begin{Lm}[Bennett's inequality \cite{bennett1962probability}]
\label{lm:bennet}
Given independent zero-mean random variables $X_1,\cdots,X_n$ satisfying the condition $|X_i|\leq a$ for all $i$, let $\sigma^2=\sum_{i=1}^n\sigma_i^2$ be the total variance where $\sigma_i^2 = Var(X_i)$. Then for any $t \geq 0$,
\begin{align*}
    P[S_n>t] \leq \exp(-\frac{\sigma^2}{a^2}h(\frac{at}{\sigma^2}))
\end{align*}
where $h(u)=(1+u)\log(1+u) - u$ and $S_n = \sum_{i=1}^n X_i$.
\end{Lm}

We now restate Theorem~3 and proceed to the main proof.
\begin{customthm}{3}
Algorithm 1 returns an $(\epsilon,\delta)$-approximation to the Shapley value with respect to $l_2$-norm if the number of tests $T$ satisfies $T\geq 8 \log \frac{N(N-1)}{2\delta} /\big((1-q_{tot}^2)h\big(\frac{\epsilon}{Zr\sqrt{N}(1-q_{tot}^2)}\big)\big) $, where $q_{tot}=\frac{N-2}{N}q(1) + \sum_{k=2}^{N-1} q(k)[1+\frac{2k(k-N)}{N(N-1)}]$, $h(u) =(1+u)\log(1+u) - u$,
$Z=2\sum _{k=1}^{N-1} \frac{1}{k}$, and $r$ is the range of the utility function.
\end{customthm}

\begin{proof} 
By Lemma~\ref{lm:shapley_diff}, the difference in Shapley values between points $i$ and $j$ is given as
\begin{align*}
    &s_i-s_j = \frac{1}{N-1}  \sum_{S\subseteq I\setminus\{i,j\}} \frac{1}{C_{N-2}^{|S|}}\bigg[U(S\cup\{i\}) - U(S\cup \{j\})\bigg]\\
   & =\frac{1}{N-1} \sum_{k=0}^{N-2} \frac{1}{C_{N-2}^{k}} \sum_{S\subseteq I\setminus\{i,j\},|S|=k}\bigg[U(S\cup\{i\}) - U(S\cup \{j\})\bigg] \, .
\end{align*}

Let $\beta_1,\cdots,\beta_N$ denote $N$ Boolean random variables drawn with the following sampler:
\begin{enumerate}
\item Sample the ``length of the sequence'' $\sum_{i=1}^N \beta_i=k \in \{1, 2,\cdots,N-1\}$, with probability $q(k)$.
\item Uniformly sample a length-$k$ sequence from $N\choose k$ all possible length-$k$ sequences
\end{enumerate}
Then the probability of any given sequence $\beta_1,\cdots,\beta_N$ is
\begin{align*}
    P[\beta_1,\cdots,\beta_N] = \frac{q(\sum_{i=1}^N \beta_i)}{C_N^{\sum_{i=1}^N \beta_i}} \, . 
\end{align*}

Now, we consider any two data points $x_i$ and $x_j$ where $i,j\in I=\{1,\cdots,N\}$ and their associated Boolean variables $\beta_i$ and $\beta_j$, and analyze
\begin{align*}
    \Delta = \beta_i U(\beta_1,\cdots,\beta_N) - \beta_j U(\beta_1,\cdots,\beta_N)
\end{align*}
Consider the expectation of $\Delta$. Obviously, only $\beta_i \not = \beta_j$ has non-zero contributions:

\begin{align*}
    &\mathbb{E}[\Delta] = \sum_{k=0}^{N-2} \frac{q(k+1)}{C_N^{k+1}} \sum_{S\subseteq I\setminus\{i,j\},|S|=k} \big[U(\beta_1,\cdots,\beta_{i-1},1,\beta_{i+1},\cdots,\beta_{j-1},0,\beta_{j+1},\cdots,\beta_N)  \nonumber\\
    &\quad \quad \quad -U(\beta_1,\cdots,\beta_{i-1},0,\beta_{i+1},\cdots,\beta_{j-1},1,\beta_{j+1},\cdots,\beta_N)\big]\\
    &=\sum_{k=0}^{N-2} \frac{q(k+1)}{C_N^{k+1}}  \sum_{S\subseteq I\setminus\{i,j\},|S|=k} \big[U(S\cup \{i\}) - U(S\cup\{j\})\big]
\end{align*}

We would like to have $Z\mathbb{E}[\Delta] = s_i-s_j$
\begin{align*}
    Z\frac{q(k+1)}{C_N^{k+1}} = \frac{1}{(N-1)C_{N-2}^k}
\end{align*}
which yields
\begin{align*}
    q(k+1) = \frac{N}{Z(k+1)(N-k-1)}= \frac{1}{Z} (\frac{1}{k+1} + \frac{1}{N-k-1})
\end{align*}
for $k=0,\cdots,N-2$. Equivalently,
\begin{align*}
    q(k) = \frac{1}{Z}(\frac{1}{k} + \frac{1}{N-k})
\end{align*}
for $k=1,\cdots,N-1$. The value of $Z$ is given by
\begin{align*}
    Z& = \sum_{k=1}^{N-1}(\frac{1}{k} + \frac{1}{N-k})=2\sum_{k=1}^{N-1}\frac{1}{k}\leq 2 (\log(N-1) + 1)
\end{align*}
Now, $\mathbb{E}[Z\Delta]=s_i-s_j$. Assume that the utility function ranges from $[0, r]$; then, we know from (\ref{eqn:shapley_diff}) that $Z\Delta$ is random variable ranges in 
$[-Zr, Zr]$.




Note that $\Delta=0$ when $\beta_i=\beta_j$. If $P[\beta_i=\beta_j]$ is large, then the variance of $\Delta$ will be much smaller than its range.

\begin{align*}
    &P[\beta_i=\beta_j] = P[\beta_i=1,\beta_j=1] + P[\beta_i = 0, \beta_j = 0]\\
    &=\bigg[ \sum_{k=2}^{N-1} \frac{q(k)}{C_N^k}C_{N-2}^{k-2}\bigg] + \bigg[\frac{N-2}{N}q(1) + \sum_{k=2}^{N-1} \frac{q(k)}{C_N^k} C_{N-2}^k\bigg]\\
    &= \frac{N-2}{N}q(1) + \sum_{k=2}^{N-1} q(k)\bigg[1+\frac{2k(k-N)}{N(N-1)}\bigg] \equiv q_{tot}
\end{align*}

Let $W = \mathbbm{1}[\Delta\neq 0]$ be an indicator of whether or not $\Delta=0$. Then, $P[W=0] = q_{tot}$ and $P[W=1] = 1-q_{tot}$. 

Now, we analyze the variance of $\Delta$. By the law of total variance,
\begin{align*}
    \texttt{Var}[\Delta] = \E[\texttt{Var}[\Delta|W]] +\texttt{Var}[\E[\Delta|W]]
\end{align*}
Recall $\Delta\in [-r,r]$. Then, the first term can be bounded by
\begin{align*}
   & \E[\texttt{Var}[\Delta|W]]= P[W=0] \texttt{Var}[\Delta|W=0] + P[W=1] \texttt{Var}[\Delta|W=1]\\
    &= q_{tot}\texttt{Var}[\Delta|\Delta = 0] + (1-q_{tot}) \texttt{Var}[\Delta|\Delta\neq 0]\\
    &= (1-q_{tot}) \texttt{Var}[\Delta|\Delta\neq 0]\\
    &\leq (1-q_{tot}) r^2
\end{align*}
where the last inequality follows from the fact that if a random variable is in the range $[m,M]$, then its variance is bounded by $\frac{(M-m)^2}{4}$.

The second term can be expressed as
\begin{align}
    &\texttt{Var}[\E[\Delta|W]]= \E_W[(\E[\Delta|W] - \E[\Delta])^2]\nonumber\\
    &= P[W=0] (\E[\Delta|W=0] - \E[\Delta])^2 + P[W=1] (\E[\Delta|W=1] - \E[\Delta])^2\nonumber\\
    &= q_{tot} (\E[\Delta|\Delta = 0] - \E[\Delta])^2+ (1-q_{tot})(\E[\Delta|\Delta\neq 0] - \E[\Delta])^2\nonumber\\
    \label{eqn:var_of_exp}
    &=q_{tot} (\E[\Delta])^2 + (1-q_{tot})(\E[\Delta|\Delta\neq 0] - \E[\Delta])^2
\end{align}

Note that
\begin{align}
    \E[\Delta] &= P[W=0] \E[\Delta|\Delta=0] + P[W=1]\E[\Delta|\Delta\neq 0]\nonumber\\
    \label{eqn:expectation_decomp}
    &=(1-q_{tot}) \E[\Delta|\Delta\neq 0]
\end{align}
Plugging (\ref{eqn:expectation_decomp}) into (\ref{eqn:var_of_exp}), we obtain 
\begin{align*}
    &\texttt{Var}[\E[\Delta|W]] =(q_{tot}(1-q_{tot})^2  + q_{tot}^2 (1-q_{tot})) (\E[\Delta|\Delta\neq 0])^2
\end{align*}
Since $|\Delta|\leq r$, $(\E[\Delta|\Delta\neq 0])^2\leq r^2$. Therefore,
\begin{align*}
     \texttt{Var}[\E[\Delta|W]] \leq q_{tot}(1-q_{tot}) r^2
\end{align*}
It follows that
\begin{align*}
    \texttt{Var}[\Delta]\leq (1-q_{tot}^2)r^2
\end{align*}

Given $T$ samples, since $\Delta_t - \E[\Delta_t] \le 2r$, the application of Bennett's inequality in Lemma~\ref{lm:bennet} yields
\begin{align*}
    &P\bigg[\sum_{t=1}^T (Z\Delta_t - \E[Z\Delta_t])> \epsilon'\bigg] \leq \exp\bigg(-\frac{T(1-q_{tot}^2)}{4} h\big(\frac{2\epsilon'}{TZr(1-q_{tot}^2)}\big)\bigg)
\end{align*}
By letting $\epsilon = \epsilon'/T$,
\begin{align*}
    & P\big[(Z\bar{\Delta} - \E[Z\Delta])> \epsilon\big]\leq \exp\bigg(-\frac{T(1-q_{tot}^2)}{4}h\big(\frac{2\epsilon}{Zr(1-q_{tot}^2)}\big)\bigg)
\end{align*}
and 
\begin{align*}
    &  P\big[ | Z\bar{\Delta} - \E[Z\Delta] |> \epsilon\big] \leq 2 \exp\bigg(-\frac{T(1-q_{tot}^2)}{4}h\big(\frac{2\epsilon}{Zr(1-q_{tot}^2)}\big)\bigg) 
\end{align*}

Therefore, the number of tests $T$ we need in order to get an $(\epsilon/(2\sqrt{N}),\delta/(N(N-1)))$-approximation to the difference of two Shapley values for a single pair of data points is
\begin{align*}
    T \geq \frac{4}{(1-q_{tot}^2)h(\frac{\epsilon}{Z\sqrt{N}r(1-q_{tot}^2)})} \log \frac{2N(N-1)}{\delta}.
\end{align*}
And the statement above holds true for any pair of $i$ and $j$.
By Lemma~\ref{lm:feasibility}, we approximate the Shapley value up to $(\epsilon,\delta)$ with $(\epsilon/(2\sqrt{N}),\delta/(N(N-1)))$ approximations to all $N(N-1)/2$ pairs of data points.

\paragraph{Complexity calculation.} It can be shown that $q_\text{tot}=1 - \frac{2}{Z}$ and so 
\begin{align}
    Z(1 - q_\text{tot}^2) = Z(1-q_\text{tot})(1+q_\text{tot}) = 2(1+ q_\text{tot})\in [2,4]
\end{align}
Therefore, as $N\rightarrow \infty$,
\begin{align}
    \frac{\epsilon}{Zr\sqrt{N}(1-q_\text{tot}^2)}\rightarrow 0
\end{align}
The Taylor expansion of $h(u)$ centered at $0$ is $\frac{u^2}{2} + \cdots$. Thus, we have
\begin{align}
    &\frac{8}{(1-q_{tot}^2)h(\frac{\epsilon}{Z\sqrt{N}r(1-q_{tot}^2)})} \log \frac{N(N-1)}{2\delta}\\
    & = \mathcal{O}\bigg(\frac{\log N}{(1-q_{tot}^2)\frac{\epsilon^2}{Z^2Nr^2(1-q_{tot}^2)^2}}\bigg)\\
    & = \mathcal{O}(NZ^2(1-q_{tot}^2)\log N)\\
    & = \mathcal{O}(NZ \log N)
\end{align}
Since $Z \leq 2(\log (N-1) + 1)$, we have $\mathcal{O}(NZ \log N) = \mathcal{O}(N(\log N)^2)$.
\end{proof}



\section{Proof of Theorem 4}

\begin{customthm}{4}
\label{thm:compress_perm}
Suppose that $U(\cdot)$ is monotone. There exists some constant $C'$ such that if $M\geq C'(K\log(N/(2K))+\log(2/\delta))$ and $T\geq \frac{2r^2}{\epsilon^2} \log \frac{4M}{\delta}$, except for an event of probability no more than $\delta$, the output of Algorithm~\ref{alg:compressive_perm} obeys
\begin{align}
        \|\hat{s}-s\|_2 \leq C_{1,K}\epsilon + C_{2,K} \frac{\sigma_K(s)}{\sqrt{K}}
\end{align}
for some constants $C_{1,K}$ and $C_{2,K}$.
\end{customthm}

\begin{proof}
Due to the monotonicity of $U(\cdot)$, $\hat{y}_{m,t}$ can be lower bounded by $
  -\frac{1}{\sqrt{M}}\sum_{i=1}^N    U(P_i^{\pi_t}\cup \{i\}) -  U(P_i^{\pi_t}) = -\frac{1}{\sqrt{M}} U(\pi_t) \geq -\frac{r}{\sqrt{M}}$; the upper bound can be similarly analyzed. Thus, the range of $\hat{y}_{m,t}$ is $[-1/\sqrt{M}r, 1/\sqrt{M}r]$. Since $\E[\hat{y}_{m,t}] = \sum_{i=1}^N A_{m,i}\E[U(P_i^{\pi_t}\cup \{i\}) -  U(P_i^{\pi_t})] = \sum_{i=1}^N A_{m,i} s_i$ for all $m=1,\ldots,M$, an application of Hoeffiding's bound gives
\begin{align}
    &P[\|As-\bar{y}\|_2\geq \epsilon] \leq P[\|As-\bar{y}\|_\infty \geq \frac{\epsilon}{\sqrt{M}}]\\
    & \leq \sum_{m=1}^M P[|A_m s - \bar{y}_m|\geq\frac{\epsilon}{\sqrt{M}}]\\
    &\leq 2M \exp(-\frac{\epsilon^2T}{2r^2})
\end{align}
Let $s = \Delta s+ \bar{s}$. Thus, $P[\|A(\bar{s} + \Delta s)-\bar{y}\|_2\leq \epsilon]$ holds with probability at least $\delta/2$ provided
\begin{align}
\label{eqn:condition_T}
     T\geq \frac{2r^2}{\epsilon^2} \log \frac{4M}{\delta}.
\end{align}

By the random matrix theory, the restricted isometry constant of $A$ satisfies $\delta_{2K}\leq C_\delta=0.465$ with probability at least $1-\delta/2$ if
\begin{align}
\label{eqn:condition_M}
    M\geq CC_\delta^{-2} (2K\log(N/(2K))+\log(2/\delta))
\end{align}
where $C>0$ is a universal constant.

Applying the Theorem 2.7 in~\cite{rauhut2010compressive}, we obtain that the output of Algorithm 2 satisfies
\begin{align}
   \|\hat{s}-s\|= \|\Delta s^*- \Delta s\| \leq C_{1,K}\epsilon + C_{2,K} \frac{\sigma_K(s)}{\sqrt{K}}
\end{align}
with probability at least $1-\delta$ provided that (\ref{eqn:condition_T}) holds and $M\geq C'(K\log(N/(2K))+\log(2/\delta))$ for some constant $C'$.
\end{proof}

\section{Proof of Theorem 5}

For the proof of Theorem 5 we need the following definition of a \emph{stable utility function}.

\begin{Df}
A utility function $U(\cdot)$ is called $\lambda$-stable if 
\begin{align*}
   \max_{i,j\in I,S\subseteq I \setminus \{i,j\}} |U(S\cup \{i\}) - U(S\cup \{j\})| \leq \frac{\lambda}{|S|+1}
\end{align*}
\end{Df}
Then, Shapley values calculated from  $\lambda$-stable utility functions have the following property.
\begin{Prop}
\label{lm:shapley_diff_bound}
If $U(\cdot)$ is $\lambda$-stable, then for all $i,j\in I$ and $i\neq j$
\begin{align*}
    s_i - s_j \leq \frac{\lambda (1+\log (N-1))}{N-1} 
\end{align*}
\end{Prop}

\begin{proof}
By Lemma~\ref{lm:shapley_diff}, we have
\begin{align*}
    &s_i-s_j \leq \frac{1}{N-1} \sum_{S\subseteq I\setminus \{i,j\}} \frac{1}{C_{N-2}^{|S|}} \frac{\lambda}{|S|+1}= \frac{1}{N-1} \sum_{|S|=0}^{N-2} \frac{\lambda}{|S|+1}
\end{align*}
Recall the bound on the harmonic sequences
\begin{align*}
    \sum_{k=1}^N \frac{1}{k} \leq 1 + \log (N)
\end{align*}
which gives us
\begin{align*}
    s_i-s_j \leq \frac{\lambda(1+\log(N-1))}{N-1}
\end{align*}
\end{proof}

Then, we can prove Theorem 5.

\begin{customthm}{5}
For a learning algorithm $A(\cdot)$ with uniform stability $\beta = \frac{C_\text{stab}}{|S|}$, where $|S|$ is the size of the training set
and $C_\text{stab}$ is some constant. Let the utility of $D$ be $U(D) = M - L_\text{test}(A(D),D_\text{test})$, where $L_\text{test}(A(D),D_\text{test}) = \frac{1}{N}\sum_{i=1}^N l(A(D),z_{\text{test},i})$ and $0\leq l(\cdot,\cdot)\leq M$. Then,
$s_i-s_j\leq 2C_\text{stab}\frac{1+\log(N-1)}{N-1}$ and the Shapley difference vanishes as $N\rightarrow \infty$. 
\end{customthm}


\begin{proof}
For any $i,j\in I$ and $i\neq j$,
\begin{align*}
 &|U(S\cup \{i\}) - U(S\cup \{j\})|\\
 & = |\frac{1}{N}\sum_{i=1}^N[l(A(S\cup\{i\}),z_{\text{test},i}) - l(A(S\cup\{j\}),z_{\text{test},i})]|\\
 &\leq \frac{1}{N}\sum_{i=1}^N | l(A(S\cup\{i\}),z_{\text{test},i}) - l(A(S),z_{\text{test},i})| +| l(A(S),z_{\text{test},i}) - l(A(S\cup\{j\}),z_{\text{test},i})| \\
 &\leq \frac{1}{N}\sum_{i=1}^N \frac{2C_\text{stab}}{|S|+1} = \frac{2C_\text{stab}}{|S|+1}
\end{align*}
Combining the above inequality with Proposition~\ref{lm:shapley_diff_bound} proves the theorem.
\end{proof}




\section{Proof of Theorem 6}
\begin{customthm}{6}
Consider the value attribution scheme that assign the value $\hat{s}(U,i) = C_U[U(S\cup \{i\})-U(S)]$ to user $i$ where $|S|=N-1$ and $C_U$ is a constant such that $\sum_{i=1}^N \hat{s}(U,i) = U(I)$. Consider two utility functions $U(\cdot)$ and $V(\cdot)$. Then, $\hat{s}(U+V,i)\neq\hat{s}(U,i) + \hat{s}(V,i)$ unless $V(I)[\sum_{i=1}^N U(S\cup\{i\})-U(S)] = U(I) [\sum_{i=1}^N V(S\cup\{i\})-V(S)]$.
\end{customthm}


\begin{proof}
Consider two utility functions $U(\cdot)$ and $V(\cdot)$. The values attributed to user $i$ under these two utility functions are given by
\begin{align*}
    \hat{s}(U,i) = C_U [U(S\cup\{i\})-U(S)]
\end{align*}
and
\begin{align*}
    \hat{s}(V,i) = C_V [V(S\cup\{i\})-V(S)]
\end{align*}
where $C_U$ and $C_V$ are constants such that $\sum_{i=1}^N\hat{s}(U,i) = U(I)$ and $\sum_{i=1}^N\hat{s}(V,i) = V(I)$. Now, we consider the value under the utility function $W(S) = U(S) +V(S)$:
\begin{align*}
    &\hat{s}(U+V,i)= C_W [U(S\cup\{i\})-U(S) + V(S\cup\{i\})-V(S)]
\end{align*}
where
\begin{align*}
    C_W = \frac{U(I)+V(I)}{\sum_{i=1}^N [U(S\cup\{i\})-U(S) + V(S\cup\{i\})-V(S)]}
\end{align*}
Then, $\hat{s}(U+V,i) =\hat{s}(U,i)+\hat{s}(V,i)$ if and only if $C_U=C_V=C_W$, which is equivalent to
\begin{align*}
  & V(I)[\sum_{i=1}^N U(S\cup\{i\})-U(S)]= U(I) [\sum_{i=1}^N V(S\cup\{i\})-V(S)] 
\end{align*}
\end{proof}

\section{Theoretical Results on the Baseline Permutation Sampling}

Let $\pi_t$ be a random permutation of $D=\{z_i\}_{i=1}^N$ and each permutation has a probability of $\frac{1}{N!}$. Let $\phi_i^t=U(P_i^{\pi_t}\cup\{i\})-U(P_i^{\pi_t})$, we consider the following estimator of $s_i$:
\begin{align*}
    \hat{s}_i = \frac{1}{T}\sum_{t=1}^T \phi_i^t
\end{align*}

\begin{Thm}
\label{thm:sampling_based}
Given the range of the utility function $r$, an error bound $\epsilon$, and a confidence $1-\delta$, the sample size required such that
\begin{align*}
    P[\|\hat{s}-s\|_2\geq \epsilon] \leq \delta
\end{align*}
is
\begin{align*}
    T\geq \frac{2r^2N}{\epsilon^2}\log \frac{2N}{\delta}
\end{align*}
\end{Thm}

\begin{proof}
\begin{align*}
   & P[\max_{i=1,\cdots,N} |\hat{s}_i-s_i|\geq \epsilon]=P[\cup_{i=1,\cdots,N} \{|\hat{s}_i-s_i|\geq \epsilon\}]  \leq \sum_{i=1}^N P[|\hat{s}_i-s_i|\geq \epsilon]\\
  &\leq 2N\exp\bigg(-\frac{2T \epsilon^2}{4r^2}\bigg)
\end{align*}
The first inequality follows from the union bound and the second one is due to Hoeffding's inequality. Since $\|\hat{s}-s\|_2\leq \sqrt{N}\|\hat{s}-s\|_\infty$, we have
\begin{align*}
    &P[\|\hat{s}-s\|_2\geq \epsilon\leq P[\|\hat{s}-s\|_\infty \geq \epsilon/\sqrt{N}]\leq 2N\exp\bigg(-\frac{2T \epsilon^2}{4Nr^2}\bigg)
\end{align*}

Setting $2N\exp(-\frac{T \epsilon^2}{2Nr^2})\leq \delta$ yields
\begin{align*}
    T\geq \frac{2r^2N}{\epsilon^2} \log \frac{2N}{\delta}
\end{align*}
\end{proof}

The permutation sampling-based method used as baseline in the experimental part of this work was adapted from Maleki et al.~\cite{maleki2013bounding} and is presented in Algorithm~\ref{alg:sampling}.

\begin{algorithm}[ht]
\SetAlgoLined
\SetKwInOut{Input}{input}
\SetKwInOut{Output}{output}
\Input{Training set - $D = \{(x_i,y_i)\}_{i=1}^N$, utility function $U(\cdot)$, the number of measurements - $M$, the number of permutations - $T$}
\Output{The Shapley value of each training point - $\hat{s}\in \R^N$}
    \For{$t\gets 1$ \KwTo $T$}{
        $\pi_t \leftarrow \text{GenerateUniformRandomPermutation}(D)$\;
        $\phi^t_i \leftarrow  U(P_i^{\pi_t}\cup \{i\}) -  U(P_i^{\pi_t})$ for $i=1,\ldots,N$\;
    }
    $\hat{s}_i = \frac{1}{T}\sum_{t=1}^T \phi^t_i$ for $i=1,\ldots,N$\;
 \caption{Baseline: Permutation Sampling-Based Approach}
 \label{alg:sampling}
\end{algorithm}

\end{document}